\newif\ifshort
\newif\ifprivate
\newif\iffinal

\documentclass{article}

\usepackage{ijcai19}

\usepackage{times}
\usepackage{soul}
\usepackage{url}
\usepackage[hidelinks]{hyperref}
\usepackage[utf8]{inputenc}
\usepackage[small]{caption}
\usepackage{graphicx}
\usepackage{amsmath}
\usepackage{booktabs}
\usepackage{algorithm}
\usepackage[noend]{algpseudocode}
\usepackage{appendix}

\usepackage[english]{babel}
\usepackage[utf8]{inputenc}
\usepackage{amsmath,amsfonts,amssymb,amsthm, bm} 
\usepackage{float}
\usepackage{array}
\usepackage{color}

\DeclareMathOperator*{\argmax}{argmax}
\DeclareMathOperator*{\IG}{IG}
\DeclareMathOperator*{\KL}{KL}
\DeclareMathOperator*{\lequal}{\leq}
\DeclareMathOperator*{\lthan}{<}

\DeclareMathOperator*{\equal}{=}
\DeclareMathOperator*{\toas}{\to}
\newcommand*{\Scale}[2][4]{\scalebox{#1}{$#2$}}%
\DeclareMathOperator{\E}{\mathbb{E}}
\DeclareMathOperator{\p}{P}

\DeclareMathOperator*{\vb}{\!\Bigm\vert\!}

\newcommand{\tagaligneq}{\refstepcounter{equation}\tag{\theequation}}

\usepackage{thmtools, thm-restate}
\declaretheorem{theorem}

\declaretheorem[sharenumber=theorem]{definition}

\ifprivate\def\phead{\textbf}\else\def\phead#1{}\fi

\allowdisplaybreaks
\title{A Strongly Asymptotically Optimal Agent in General Environments}

\author{
Michael K.\ Cohen$^1$\footnote{Contact Author}\and
Elliot Catt$^1$\And
Marcus Hutter$^1$\\
\affiliations
$^1$Australian National University\\
\emails
\{michael.cohen, elliot.carpentercatt, marcus.hutter\}@anu.edu.au
}

\begin{document}

\sloppy

\maketitle

\begin{abstract}
Reinforcement Learning agents are expected to eventually perform well. Typically, this takes the form of a guarantee about the asymptotic behavior of an algorithm given some assumptions about the environment. We present an algorithm for a policy whose value approaches the optimal value with probability 1 in all computable probabilistic environments, provided the agent has a bounded horizon. This is known as strong asymptotic optimality, and it was previously unknown whether it was possible for a policy to be strongly asymptotically optimal in the class of all computable probabilistic environments. Our agent, Inquisitive Reinforcement Learner (Inq), is more likely to explore the more it expects an exploratory action to reduce its uncertainty about which environment it is in, hence the term inquisitive. Exploring inquisitively is a strategy that can be applied generally; for more manageable environment classes, inquisitiveness is tractable. We conducted experiments in ``grid-worlds'' to compare the Inquisitive Reinforcement Learner to other weakly asymptotically optimal agents.

\end{abstract}


\section{Introduction}
\begin{quote}
    \textit{``Efforts to solve [an instance of the exploration-exploitation problem] so sapped the energies and minds of Allied analysts that the suggestion was made that the problem be dropped over Germany, as the ultimate instrument of intellectual sabotage.''} --Peter Whittle \cite{whittle_1979}
\end{quote}
 The Allied analysts were considering the simplest possible problem in which there is a trade-off to be made between exploiting, taking the apparently best option, and exploring, choosing a different option to learn more. We tackle what we consider the most difficult instance of the exploration-exploitation trade-off problem: when the environment could be any computable probability distribution, not just a multi-armed bandit, how can one achieve optimal performance in the limit? 

Our work is within the Reinforcement Learning (RL) paradigm: an agent selects an action, and the environment responds with an observation and a reward. The interaction may end, or it may continue forever. Each interaction cycle is called a timestep. The agent has a discount function that weights its relative concern for the reward it achieves at various future timesteps. The agent’s job is to select actions that maximize the total expected discounted reward it achieves in its lifetime. The ``value” of an agent’s policy at a certain point in time is the expected total discounted reward it achieves after that time if it follows that policy. One formal specification of the exploration-exploitation problem is: what policy can an agent follow so that the policy’s value approaches the value of the optimal informed policy with probability 1, even when the agent doesn’t start out knowing the true dynamics of its environment?

Most work in RL makes strong assumptions about the environment—that the environment is Markov, for instance. Impressive recent development in the field of reinforcement learning often makes use of the Markov assumption, including Deep Q Networks \cite{mnih2015human}, A3C \cite{mnih2016asynchronous}, Rainbow \cite{hessel_2018}, and AlphaZero \cite{silver2017mastering}. Another example of making strong assumptions in RL comes from some model-based algorithms that implicitly assume that the environment is representable by, for example, a fixed-size neural network, or whatever construct is used to model the environment. We do not make any such assumptions.

Many recent developments in RL are largely about tractably learning to exploit; how to explore intelligently is a separate problem. We address the latter problem. Our approach, inquisitiveness, is based on Orseau et al.’s \shortcite{Hutter:13ksaprob} Knowledge Seeking Agent for Stochastic Environments, which selects the actions that best inform the agent about what environment it is in. Our Inquisitive Reinforcement Learner (Inq) explores like a knowledge seeking agent, and is more likely to explore when there is apparently (according to its current beliefs) more to be learned. Sometimes exploring well requires ``expeditions,” or many consecutive exploratory actions. Inq entertains expeditions of all lengths, although it follows the longer ones less often, and it doesn't resolutely commit in advance to seeing the expedition through.

This is a very human approach to information acquisition. When we spot an opportunity to learn something about our natural environment, we feel inquisitive. We get distracted. We are inclined to check it out, even if we don’t see directly in advance how this information might help us better achieve our goals. Moreover, if we can tell that the opportunity to learn something requires a longer term project, we may find ourselves less inquisitive.

For the class of computable environments (stochastic environments that follow a computable probability distribution), it was previously unknown whether any policy could achieve strong asymptotic optimality (convergence of the value to optimality with probability 1). Lattimore et al. \shortcite{Hutter:11asyoptag} showed that no deterministic policy could achieve this. The key advantage that stochastic policies have is that they can let the exploration probability go to $0$ while still exploring infinitely often. (For example, an agent that explores with probability $1/t$ at time $t$ still explores infinitely often).

There is a weaker notion of optimality--``weak asymptotic optimality''--for which positive results already exist; this condition requires that the average value over the agent's lifetime approach optimality. Lattimore et al. \shortcite{Hutter:11asyoptag} identified a weakly asymptotically optimal agent for \textit{deterministic} computable environments; the agent maintains a list of environments consistent with its observations, exploiting as if it is in the first such one, and exploring in bursts. A recent algorithm for a Thompson Sampling Bayesian agent was shown, with an elegant proof, to be weakly asymptotically optimal in all computable environments, but not strongly asymptotically optimal \cite{Hutter:16thompgrl}. 

Most work in RL regards (Partially Observable) Markov Decision Processes (PO)MDPs. However, environments that enter completely novel states infinitely often render (PO)MDP algorithms helpless. For example, an RL agent acting as a chatbot, optimizing a function, or proving mathematical theorems would struggle to model the environment as an MDP, and would likely require an exploration mechanism like ours. In the chatbot case, for instance, as a conversation with a person progresses, the person never returns to the same state.

If we formally compare Inq to existing algorithms in MDPs, we find that many achieve asymptotic optimality. Epsilon-greedy, upper confidence bound, and Thompson sampling exploration strategies suffice in MDPs. Our primary motivation is for the sorts of environments described above. To discriminate between exploratory approaches in \textit{ergodic} MDPs, one can formally bound regret, and we would like to do this for Inq in the future.

For comparison, some algorithms which use the MDP formalism also consider information-theoretic approaches to exploration, such as VIME \cite{houthooft2016curiosity}, the agent in \cite{still2009information}, and TEXPLORE-VANIR \cite{hester2012intrinsically}.

In Section 2, we formally describe the RL setup and present notation. In Section 3, we present the algorithm for Inq. In Section 4, we prove our main result: that Inq is strongly asymptotically optimal. In Section 5, we present experimental results comparing Inq to weakly asymptotically optimal agents. Finally, we discuss the relevance of this exploration regime to tractable algorithms. Appendix \ref{app:notation} collates notation and definitions for quick reference. Appendix \ref{app:proofs} contains the proofs of the lemmas.

\section{Notation}

We follow the notation of Orseau, et al. \shortcite{Hutter:13ksaprob}. The reinforcement learning setup is as follows: $\mathcal{A}$ is a finite set of actions available to the agent; $\mathcal{O}$ is a finite set of observations it might observe, and $\mathcal{R} = [0, 1] \cap \mathbb{Q}$ is the set of possible rewards. The set of all possible interactions in a timestep is $\mathcal{H} := \mathcal{A} \times \mathcal{O} \times \mathcal{R}$. At every timestep, one element from this set occurs. A reinforcement learner's policy $\pi$ is a stochastic function which outputs an action given an interaction history, denoted by $\pi: \mathcal{H}^* \rightsquigarrow \mathcal{A}$. ($\mathcal{X}^* := \bigcup_{i=0}^\infty \mathcal{X}^i$ represents all finite strings from an alphabet $\mathcal{X}$). An environment is a stochastic function which outputs an observation and reward given an interaction history and an action: $\nu: \mathcal{H}^* \times \mathcal{A} \rightsquigarrow \mathcal{O} \times \mathcal{R}$. For a stochastic function $f: \mathcal{X} \to \mathcal{Y}$, $f(y|x)$ denotes the probability that $f$ outputs $y \in \mathcal{Y}$ when $x \in \mathcal{X}$ is input.

A policy and an environment induce a probability measure over $\mathcal{H}^\infty$, the set of all possible infinite histories: for $h \in \mathcal{H}^*$, $\p^\pi_\nu(h)$ denotes the probability that an infinite history begins with $h$ when actions are sampled from the policy $\pi$, and observations and rewards are sampled from the environment $\nu$. Formally, we define this inductively: $\p^\pi_\nu(\epsilon) \mapsto 1$, where $\epsilon$ is the empty history, and for $h \in \mathcal{H}^*$, $a \in \mathcal{A}$, $o \in \mathcal{O}$, $r \in \mathcal{R}$, we define $\p^\pi_\nu(haor) \mapsto \p^\pi_\nu(h)\pi(a|h)\nu(or|ha)$. In an infinite history $h_{1:\infty} \in \mathcal{H}^\infty$, $a_t$, $o_t$, and $r_t$ refer to the $t$th action, observation and reward, and $h_t$ refers to the $t$th timestep: $a_t o_t r_t$. $h_{<t}$ refers to the first $t-1$ timesteps, and $h_{t:k}$ refers to the string of timesteps $t$ through $k$ (inclusive). Strings of actions, observations, and rewards are notated similarly.

A Bayesian agent deems a class of environments a priori feasible. Its ``beliefs'' take the form of a probability distribution over which environment is the true one. We call this the agent's belief distribution. In our formulation, Inq considers any computable environment feasible, and starts with a prior belief distribution based on the environments' Kolmogorov complexities: that is, the length of the shortest program that computes the environment on some reference machine. However, all our results hold as long as the true environment is contained in the class of environments that are considered feasible, and as long as the prior belief distribution assigns nonzero probability to each environment in the class. We take $\mathcal{M}$ to be the class of all computable environments, and $w(\nu) := 2^{-K(\nu)(1+\varepsilon)}/\mathcal{N}$ to be the prior probability of the environment $\nu$, where $K$ is the Kolmogorov complexity, $\varepsilon > 0$, and $\mathcal{N}$ is a normalization constant. ($\varepsilon > 0$ ensures the prior has finite entropy, which facilitates analysis.) A smaller class with a different prior probability could easily be substituted for $\mathcal{M}$ and $w(\nu)$.

We use $\xi$ to denote the agent's beliefs about future observations. Together with a policy $\pi$ it defines a Bayesian mixture measure: $\p^\pi_\xi(\cdot) := \sum_{\nu \in \mathcal{M}} w(\nu) \p^\pi_\nu(\cdot)$. The posterior belief distribution of the agent after observing a history $h \in \mathcal{H}^*$ is $w(\nu|h) := w(\nu) \p^{\pi'}_\nu(h)/\p^{\pi'}_\xi(h)$. This definition is independent of the choice of $\pi'$ as long as $\p^{\pi'}_\xi(h) > 0$; we can fix a reference policy $\pi'$ just for this definition if we like. We sometimes also refer to the conditional distribution $\xi(or|ha) := \sum_{\nu \in \mathcal{M}} w(\nu|h) \nu(or|ha)$.

The agent's discount at a timestep is denoted $\gamma_t$. To normalize the agent's policy's value to $[0, 1]$, we introduce $\Gamma_t := \sum_{k=t}^\infty \gamma_k$. (Normalization makes value convergence nontrivial). We consider an agent with a bounded horizon: $\forall \varepsilon > 0 \ \exists m \ \forall t : \Gamma_{t+m} / \Gamma_t \leq \varepsilon$. Intuitively, this means that the agent does not become more and more farsighted over time. Note this does not require a finite horizon. A classic discount function giving a bounded horizon is a geometric one: for $0 \leq \gamma < 1$, $\gamma_t = \gamma^t$. The value of a policy $\pi$ in an environment $\nu$, given a history $h_{<t} \in \mathcal{H}^{t-1}$, is
\begin{equation}
    V^\pi_\nu(h_{<t}) := \frac{1}{\Gamma_t}\mathbb{E}^\pi_\nu\left[\sum_{k=t}^\infty \gamma_k r_k\Biggm\vert h_{<t}\right]
\end{equation}
Here, the expectation is with respect to the probability measure $\p^\pi_\nu$. Reinforcement Learning is the attempt to find a policy that makes this value high, without access to $\nu$.

\section{Inquisitive Reinforcement Learner}
We first describe how Inq exploits, then how it explores. It exploits by maximizing the discounted sum of its reward in expectation over its current beliefs, and it explores by following maximally informative ``exploratory expeditions'' of various lengths.

An optimal policy with respect to an environment $\nu$ is a policy that maximizes the value.
\begin{equation}
    \pi^*_{\nu}(\cdot) := \argmax_{\pi \in \Pi} V^\pi_\nu(\cdot)
\end{equation}
where $\Pi = \mathcal{H}^* \rightsquigarrow \mathcal{A}$ is the space of all policies. An optimal deterministic policy always exists \cite{Hutter:14tcdiscx}. When exploiting, Inq simply maximizes the value according to its belief distribution $\xi$. Since this policy is deterministic, we write $a^*(h_{<t})$ to mean the unique action at time $t$ for which $\pi_\xi^*(a | h_{<t}) = 1$. That is the exploitative action.

The most interesting feature of Inq is how it gets distracted by the opportunity to explore. Inq explores to learn. An agent has learned from an observation if its belief distribution $w$ changes significantly after making that observation. If the belief distribution has hardly changed, then the observation was not very informative. The typical information-theoretic measure for how well a distribution $Q$ approximates a distribution $P$ is the KL-divergence, $\KL(P||Q)$. Thus, a principled way to quantify the information that an agent gains in a timestep is the KL-divergence from the belief distribution at time $t+1$ to the belief distribution at time $t$. This is the rationale behind the construction of Orseau, et al.'s \shortcite{Hutter:13ksaprob} Knowledge Seeking Agent, which maximizes this expected information gain.

Letting $h_{<t} \in \mathcal{H}^{t-1}$ and $h' \in \mathcal{H}^*$, the information gain at time $t$ is defined: 
\begin{equation}
    \IG(h'|h_{<t}) := \sum_{\nu \in \mathcal{M}} w(\nu|h_{<t}h')\log \frac{w(\nu|h_{<t}h')}{w(\nu|h_{<t})}
\end{equation}
Recall that $w(\nu|h)$ is the posterior probability assigned to $\nu$ after observing $h$.

An $m$-step expedition, denoted $\alpha^m$, represents all contingencies for how an agent will act for the next $m$ timesteps. It is a deterministic policy that takes history-fragments of length less than $m$ and returns an action:
\begin{equation}
    \alpha^m : \bigcup_{i=0}^{m-1} \mathcal{H}^i \to \mathcal{A}
\end{equation}
$\p^{\alpha^m}_\xi(h_{<t+k}|h_{<t})$ is a conditional distribution defined for $0 \leq k \leq m$, which represents the conditional probability of observing $h_{<t+k}$ if the expedition $\alpha^m$ is followed starting at time $t$, after observing $h_{<t}$. Now we can consider the information-gain value of an $m$-step expedition. It is the expected information gain upon following that expedition:
\begin{multline}
	V^{\IG}(\alpha^m, h_{<t}) := \\ \sum_{h_{t:t+m-1}\in \mathcal{H}^m} \p^{\alpha^m}_\xi(h_{<t+m}|h_{<t}) \IG(h_{t:t+m-1}|h_{<t})	
\end{multline}

At a time $t$, one might consider many expeditions: the one-step expedition which maximizes expected information gain, the two-step expedition doing the same, etc. Or one might consider carrying on with an expedition that began three timesteps ago.

\begin{definition}
	At time $t$, the $m$-$k$ expedition is the $m$-step expedition beginning at time $t-k$ which maximized the expected information gain from that point.\footnote{Ties in the argmax are broken arbitrarily.}
\begin{equation}
	 \alpha^{\IG}_{m, k}(h_{<t}) := \argmax_{\alpha^m \ : \ \bigcup_{i=0}^{m-1} \mathcal{H}^i \to \mathcal{A}} V^{\IG}(\alpha^m, h_{<t-k})
\end{equation}
\end{definition}


Example expeditions are diagrammed in Figure \ref{fig:expeditions}.

\begin{figure}
    \centering
    \includegraphics[width=0.8\linewidth]{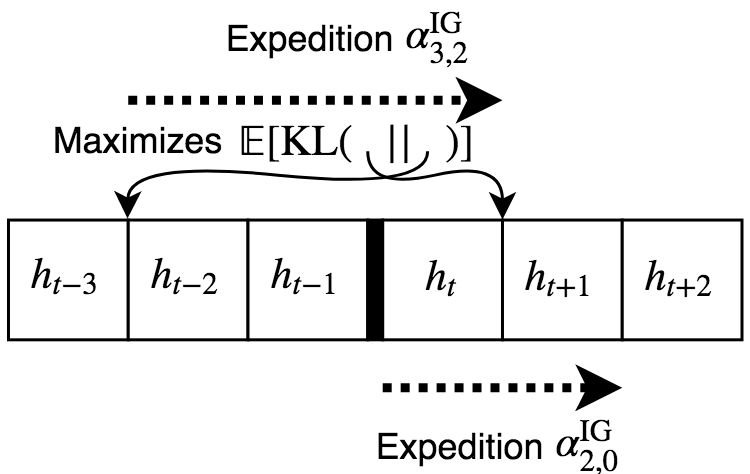}
    \caption{\textbf{Example Expeditions.} Expeditions maximize the expected KL-divergence from the posterior at the end to the posterior at the beginning.}
    \label{fig:expeditions}
\end{figure}

Expeditions are functions which return an action given what has been seen so far on the expedition. The $m$-$k$ exploratory action is the action to take at time $t$ according to the $m$-$k$ expedition:
\begin{equation}
    a^{\IG}_{m, k}(h_{<t}) := \alpha^{\IG}_{m, k}(h_{<t})(h_{t-k:t-1})
\end{equation}
Naturally, this is only defined for $k < m,t$, since the expedition function can't accept a history fragment of length $\geq m$, and $t-k$ must be positive. Note also that if $k=0$, $h_{t-k:t-1}$ evaluates to the empty string, $\epsilon$.

The reason Inq doesn't ignore expeditions that started in the past is that Inq must have some chance of actually executing the whole expedition (for every expedition). If the probability of completing an expedition is 0, one cannot use it for a bound on Inq's belief-accuracy.

\begin{definition}
Let $\rho (h_{<t},m,k)$ be the probability of taking the $m$-$k$ exploratory action after observing a history $h_{<t}$.
	\begin{multline}
				\rho (h_{<t},m,k) := \\ 
				\min\left\{\frac{1}{m^2(m+1)}, \eta V^{\IG}(\alpha^{\IG}_{m, k}(h_{<t}), h_{<t-k})\right\}
	\end{multline}
	where $\eta$ is an exploration constant.
\end{definition}

Note in the definition of $\rho (h_{<t},m,k)$ that the probability of following an expedition goes to $0$ if the expected information gain from that expedition goes to $0$. The first term in the $\min$ ensures the probabilities will not sum to more than 1. The total probability of exploration is defined: \begin{multline}
	 \beta(h_{<t}) := \sum_{m \in \mathbb{N}} \sum_{k<m,t} \rho(h_{<t},m,k) \leq \\
	 \sum_{m \in \mathbb{N}} \sum_{k<m,t} \frac{1}{m^2(m+1)} \leq \sum_{m \in \mathbb{N}} \sum_{k<m} \frac{1}{m^2(m+1)} = 1
\end{multline}

The feature that makes Inq inquisitive is that $\rho (h_{<t},m,k)$ is proportional to the expected information gain from the $m$-$k$ expedition, $V^{\IG}(\alpha^{\IG}_{m, k}(h_{<t}), h_{<t-k})$. 
 Note that completing an $m$-step expedition requires randomly deciding to explore in that way on $m$ separate occasions. While this may seem inefficient, if the agent always got boxed into long expeditions, the value of its policy would plummet infinitely often.

Finally, Inq's policy $\pi^\dagger$, defined in Algorithm \ref{alg:inq_agent}, takes the $m$-$k$ exploratory action with probability $\rho (\cdot,m,k)$, and takes the exploitative action otherwise.\footnote{This algorithm is written in a simplified way that does not halt, but if a real number in $[0, 1]$ is sampled first, the actions can be assigned to disjoint intervals successively until the sampled real number lands in one of them.}

\begin{algorithm}
\caption{Inquisitive Reinforcement Learner's Policy $\pi^\dagger$}\label{alg:inq_agent}
\begin{algorithmic}[1]
\While{True}
\State calculate $\rho (h_{<t},m,k)$ for all $m$ and for all $k < \min\{m, t\}$
\State take action $a^{\IG}_{m, k}(h_{<t})$ with probability $\rho (h_{<t},m,k)$
\State take action $a^*(h_{<t})$ with probability $1 - \beta(h_{<t})$

\EndWhile
\end{algorithmic}
\end{algorithm}


\section{Strong Asymptotic Optimality}

Here we present our central result: that the value of $\pi^\dagger$ approaches the optimal value. We present the theorem, motivate the result, and proceed to the proof. We recommend the reader have Appendix \ref{app:notation} at hand for quickly looking up definitions and notation.

Before presenting the theorem, we clarify an assumption, and define the optimal value. We call the true environment $\mu$, and we assume that $\mu \in \mathcal{M}$. For $\mathcal{M}$ the class of computable environments, this is a very unassuming assumption. The optimal value is simply the value of the optimal policy with respect to the true environment:
\begin{equation}
    V^*_\mu(h_{<t}) := \sup_{\pi \in \Pi} V^\pi_\mu(h_{<t}) = V^{\pi^*_{\mu}}_\mu(h_{<t})
\end{equation}

Recall also that we have assumed the agent has a bounded horizon in the sense that $\forall \varepsilon \ \exists m \ \forall t : \Gamma_{t+m}/\Gamma_t \leq \varepsilon$. The Strong Asymptotic Optimality theorem is that under these conditions, the value of Inq's policy approaches the optimal value with probability 1, when actions are sampled from Inq's policy and observations and rewards are sampled from the true environment $\mu$.

\begin{theorem}[Strong Asymptotic Optimality] \label{thm:sao}
As $t \to \infty$, 
\begin{equation*}
    V^*_\mu(h_{<t}) - V^{\pi^\dagger}_\mu(h_{<t}) \to 0 \ \ \textrm{with }\p^{\pi^\dagger}_\mu \textrm{\!\!-prob. 1}
\end{equation*}
where $\mu \in \mathcal{M}$ is the true environment.
\end{theorem}

For a Bayesian agent, uncertainty about on-policy observations goes to $0$. Since ``on-policy'' for Inq includes, with some probability, all maximally informative expeditions, Inq eventually has little uncertainty about the result of any course of action, and can therefore successfully select the optimal course. For any fixed horizon, Inq's mixture measure $\xi$ approaches the true environment $\mu$.

We use the following notation for a particular KL-divergence that plays a central role in the proof:
\begin{equation}
    \KL_{h_{<t}, n}(\p^\pi_{\nu_1} || \p^\pi_{\nu_2}) := \sum_{h' \in \mathcal{H}^n} \p^\pi_{\nu_1}(h'|h_{<t}) \log \frac{\p^\pi_{\nu_1}(h'|h_{<t})}{\p^\pi_{\nu_2}(h'|h_{<t})}
\end{equation}
This quantifies the difference between the expected observations of two different environments that would arise in the next $n$ timesteps when following policy $\pi$. $\KL_{h_{<t}, \infty}$ denotes the limit of the above as $n \to \infty$, which exists by \cite[proof of Theorem 3]{Hutter:13ksaprob}.

In dealing with the KL-divergence, we simplify matters by asserting that $0 \log 0 := 0$, and $0 \log \frac{0}{0} := 0$.

We begin with a lemma that equates the information gain value of an expedition with the expected prediction error. The KL-divergence on the right hand side represents how different $\nu$ and $\xi$ appear when following the expedition in question.

\begin{restatable}{lemma}{lemorseau}\label{lemorseau4}
\begin{equation*}
V^{\IG}(\alpha^m, h_{<t}) = \sum_{\nu \in \mathcal{M}} w(\nu|h_{<t}) \KL_{h_{<t}, m} \left( \p^{\alpha^m}_\nu \Bigm\vert \Bigm\vert \p^{\alpha^m}_\xi\right)	
\end{equation*}
\end{restatable}

Proofs of Lemmas appear in Appendix \ref{app:proofs}.

Recall that $w(\nu|h_{<t})$ is the posterior weight that Inq assigns to the environment $\nu$ after observing $h_{<t}$. We show that the infimum of this value is strictly positive with probability 1.

\begin{restatable}{lemma}{lemwmuinf}\label{lemwmuinf}
$\inf_t w(\mu|h_{<t}) > 0 \ \ \textrm{w.$\p^\pi_\mu$-p. 1}$
\end{restatable}

Next, we show that every exploration probability $\rho(h_{<t}, m, k)$ goes to $0$.
From here, all ``w.p.1'' statements mean with $\p^{\pi^\dagger}_\mu$-probability 1, if not otherwise specified.

\begin{restatable}{lemma}{lemexppart}\label{lemexppart0}
\begin{equation*}
	\rho(h_{<t}, m, k) \toas^{t \to \infty} 0 \ \ \mathrm{w.p. 1}
\end{equation*} 
\end{restatable}

The essence of the proof is that with a finite-entropy prior, there is only a finite amount of information to gain, so the expected information gain (and the exploration probability) goes to 0.

Next, we show that the total exploration probability goes to 0:
\begin{restatable}{lemma}{lemexp}\label{lemexp0}
\begin{equation*}
	\beta(h_{<t}) \to 0 \ \ \mathrm{w.p. 1}
\end{equation*} 
\end{restatable}

Lemma \ref{lempredconv} shows that the probabilities assigned by $\xi$ converge to those of $\mu$.

\begin{restatable}{lemma}{lempredconv}\label{lempredconv}
$\forall m \in \mathbb{N}$, $h_{t:t+m-1} \in \mathcal{H}^m$, $\alpha^m : \ \bigcup_{i=0}^{m-1} \mathcal{H}^i \to \mathcal{A}$: 
\begin{equation*}
	\p^{\alpha^m}_\mu(h_{t:t+m-1}|h_{<t}) - \p^{\alpha^m}_\xi(h_{t:t+m-1}|h_{<t}) \toas^{t \to \infty} 0 \ \ \mathrm{w.p.1}
\end{equation*}
\end{restatable}

The proof of Lemma~\ref{lempredconv} roughly follows the following argument: if all exploration probabilities go to $0$, then the informativeness of the maximally informative expeditions goes to 0, so the informativeness of all expeditions goes to 0, meaning the prediction error goes to 0.

Finally, we prove the Strong Asymptotic Optimality Theorem: $V^*_\mu(h_{<t}) - V^{\pi^\dagger}_\mu(h_{<t}) \to 0 \ \ \textrm{with }\p^{\pi^\dagger}_\mu \textrm{\!\!-prob. 1}$.

\begin{proof}[Proof of Theorem \ref{thm:sao}]
Let $\varepsilon > 0$. Since the agent has a bounded horizon, there exists an $m$ such that for all $t$, $\frac{\Gamma_{t+m}}{\Gamma_{t}} \leq \varepsilon$. Recall
\begin{equation}
    V^*_\mu(h_{<t}) = \frac{1}{\Gamma_t}\mathbb{E}^{\pi^*_{\mu}}_\mu\left[\sum_{k=t}^\infty \gamma_k r_k \biggm\vert h_{<t}\right]
\end{equation}
Using the $m$ from above, let
\begin{equation}
    V^{*\setminus m}_\mu(h_{<t}) := \frac{1}{\Gamma_t}\mathbb{E}^{\pi^*_{\mu}}_\mu\left[\sum_{k=t}^{t+m-1} \gamma_k r_k \biggm\vert h_{<t}\right]
\end{equation}
Since $r_t \in [0, 1]$,
\begin{equation}
    |V^*_\mu(h_{<t}) - V^{*\setminus m}_\mu(h_{<t})| \leq \frac{\Gamma_{t+m}}{\Gamma_{t}} \leq \varepsilon
\end{equation}

We continue from there:

\begin{align*}
&V^{*}_\mu(h_{<t}) \\
&\leq V^{*\setminus m}_\mu(h_{<t}) + \varepsilon
\\
&\equal \frac{1}{\Gamma_t} \sum_{h_{t:t+m-1} \in \mathcal{H}^{m}} \p^{\pi^*_{\mu}}_\mu(h_{t:t+m-1}|h_{<t}) \sum_{k=t}^{t+m-1} \gamma_k r_k + \varepsilon
\\
&\Scale[0.8]{(a)}
\\[-9pt]
&\hspace{-5.5mm}\lequal^{\exists T_1 \ \forall t > T_1} \frac{1}{\Gamma_t} \sum_{h_{t:t+m-1} \in \mathcal{H}^{m}} \p^{\pi^*_{\mu}}_\xi(h_{t:t+m-1}|h_{<t}) \sum_{k=t}^{t+m-1} \gamma_k r_k+ 2\varepsilon  
\\
&\lequal^{(b)} \frac{1}{\Gamma_t} \mathbb{E}^{\pi^*_{\mu}}_\xi \left[\sum_{k=t}^{\infty} \gamma_k r_k  \biggm\vert h_{<t}\right] + 2\varepsilon
\\
&\lequal^{(c)} \frac{1}{\Gamma_t} \mathbb{E}^{\pi^*_{\xi}}_\xi \left[\sum_{k=t}^{\infty} \gamma_k r_k  \biggm\vert h_{<t}\right] + 2\varepsilon
\\
&\lequal^{(d)} \frac{1}{\Gamma_t} \sum_{h_{t:t+m-1} \in \mathcal{H}^{m}} \p^{\pi^*_{\xi}}_\xi(h_{t:t+m-1}|h_{<t}) \sum_{k=t}^{t+m-1} \gamma_k r_k + 3\varepsilon  
\\
&\Scale[0.8]{(e)}
\\[-9pt]
&\hspace{-5.5mm}\lequal^{\exists T_2 \ \forall t > T_2} \frac{1}{\Gamma_t} \sum_{h_{t:t+m-1} \in \mathcal{H}^{m}} \p^{\pi^*_{\xi}}_\mu(h_{t:t+m-1}|h_{<t}) \sum_{k=t}^{t+m-1} \gamma_k r_k + 4\varepsilon
\\
&\Scale[0.8]{(f)}
\\[-9pt]
&\hspace{-5.5mm}\lequal^{\exists T_3 \ \forall t > T_3}\frac{1}{\Gamma_t} \sum_{h_{t:t+m-1} \in \mathcal{H}^{m}} \frac{\p^{\pi^\dagger}_\mu(h_{t:t+m-1}|h_{<t})}{\prod_{k=t}^{t+m-1} (1 - \beta(h_{<k}))} \sum_{k=t}^{t+m-1} \gamma_k r_k \\
&\ \hspace{7mm}+ 4\varepsilon 
\\
&\lequal \frac{1}{\Gamma_t} \sum_{h_{t:t+m-1} \in \mathcal{H}^{m}} \frac{\p^{\pi^\dagger}_\mu(h_{t:t+m-1}|h_{<t})}{(1 - \max_{t \leq k < t + m} \beta(h_{<k}))^{m}} \sum_{k=t}^{t+m-1} \gamma_k r_k \\
&\ \hspace{7mm} + 4\varepsilon 
\\
&\Scale[0.8]{(g)}
\\[-9pt]
&\hspace{-12.5mm}\lequal^{\hspace{7mm}\exists T_4, \varepsilon' > 0 \ \forall t > T_4} \frac{1}{\Gamma_t} \sum_{h_{t:t+m-1} \in \mathcal{H}^{m}} \frac{\p^{\pi^\dagger}_\mu(h_{t:t+m-1}|h_{<t})}{(1 - \varepsilon')^{m}} \sum_{k=t}^{t+m-1} \gamma_k r_k \\
&\ \hspace{7mm} + 4\varepsilon 
\\
&\lequal^{(h)} \frac{1}{(1-\varepsilon')^m \Gamma_t} \mathbb{E}^{\pi^\dagger}_\mu \left[\sum_{k=t}^{\infty} \gamma_k r_k \biggm\vert h_{<t}\right] + 4\varepsilon
\\
&= \frac{1}{(1-\varepsilon')^m} V^{\pi^\dagger}_\mu(h_{<t}) + 4\varepsilon
\\
&= V^{\pi^\dagger}_\mu(h_{<t}) + 4\varepsilon + (\frac{1}{(1-\varepsilon')^m} - 1) V^{\pi^\dagger}_\mu(h_{<t})
\\
&\lequal^{(i)} V^{\pi^\dagger}_\mu(h_{<t}) + 4\varepsilon + (\frac{1}{(1-\varepsilon')^m} - 1)
\tagaligneq
\end{align*}

$(a)$, $(e)$, $(f)$, and $(g)$ all hold with probability 1. $(a)$ follows from Lemma \ref{lempredconv}: for all $m$, $\p^\pi_\xi(\cdot | h_{<t}) \to \p^\pi_\mu(\cdot | h_{<t})$ for all conditional probabilities of histories of length $m$, with probability 1, and the countable sum is bounded (by $\Gamma_t$). $(b)$ follows from adding more non-negative terms to the sum. $(c)$ follows $\pi^*_\xi$ being the $\xi$-optimal policy, and therefore it accrues at least as much expected reward in environment $\xi$ as $\pi^*_\mu$ does. $(d)$ follows from $\sum_{k = t+m}^\infty \gamma_k / \Gamma_t = \Gamma_{t+m}/\Gamma_t \leq \varepsilon$, and $r_t \in [0, 1]$. $(e)$ follows from Lemma \ref{lempredconv} just as $(a)$ did. $(f)$ follows because the product in the denominator is the probability that $\pi^\dagger$ mimics $\pi^*_\xi$ for $m$ consecutive timesteps, and by Lemma \ref{lemexp0} there is a time after which this probability is uniformly strictly positive. $(g)$ follows from Lemma \ref{lemexp0}: $\beta(h_{<k}) \to 0$ with probability 1. $(h)$ follows from adding more non-negative terms to the sum. Finally, $(i)$ follows from the value being normalized to $[0, 1]$ by $\Gamma_t$.

$\forall \delta > 0 \ \exists \varepsilon > 0, \varepsilon' > 0: 4\varepsilon + (\frac{1}{(1 - \varepsilon')^{m}} - 1) < \delta$. Letting $T = \max\{T_1, T_2, T_3, T_4\}$, we can combine the equations above to give
\begin{equation}
    \forall \delta > 0 \ \exists T \ \forall t > T : V^{*}_\mu(h_{<t}) - V^{\pi^\dagger}_\mu(h_{<t}) < \delta \ \ \ \mathrm{w.p.1}
\end{equation}
Since $V^{*}_\mu(h_{<t}) \geq V^{\pi^\dagger}_\mu(h_{<t})$, 
\begin{equation}
    V^{*}_\mu(h_{<t}) - V^{\pi^\dagger}_\mu(h_{<t}) \to 0 \ \ \mathrm{w.p.1}
\end{equation}
\end{proof}

Strong Asymptotic Optimality is not a guarantee of efficacy; consider an agent that ``commits suicide'' on the first timestep, and thereafter receives a reward of $0$ no matter what it does. This agent is asymptotically optimal, but not very useful. In general, when considering many environments with many different ``traps,'' bounded regret is impossible to guarantee \cite{Hutter:04uaibook}, but one can still demand from a reinforcement learner that it make the best of whatever situation it finds itself in by correctly identifying (in the limit) the optimal policy.

We suspect that strong asymptotic optimality would not hold if Inq had an unbounded horizon, since its horizon of concern may grow faster than it can learn about progressively more long-term dynamics of the environment. Going more into the technical details, let $\Delta_{kt}$ be, roughly ``at time $t$, how much does $\xi$ differ from $\mu$ regarding predictions about the next $k$ timesteps?'' A lemma in our proof is that $\forall k \ \lim_{t \to \infty} \Delta_{kt} = 0$, but this does not imply, for example, that $\lim_{z \to \infty} \Delta_{zz} = 0$. If the horizon which is necessary to predict is growing over time, Inq might not be strongly asymptotically optimal.

Indeed, we tenuously suspect that it is impossible for an agent with an unbounded time horizon to be strongly asymptotically optimal in the class of all computable environments. If that is true, then the assumptions that our result relies on (namely that the true environment is computable, and the agent has a bounded horizon) are the bare minimum for strong asymptotic optimality to be possible.

Inq is not computable; in fact, no computable policy can be strongly asymptotically optimal in the class of all computable environments (Lattimore, et al. \shortcite{Hutter:11asyoptag} show this for deterministic policies, but a simple modification extends this to stochastic policies). For many smaller environment classes, however, Inq would be computable, for example if $\mathcal{M}$ is finite, and perhaps for decidable $\mathcal{M}$ in general. The central result, that inquisitiveness is an effective exploration strategy, applies to any Bayesian agent.


\section{Experimental Results}

We compared Inq with other known weakly asymptotically optimal agents, Thompson sampling and BayesExp \cite{lattimore2014bayesian}, in the grid-world environment using AIXIjs \cite{aslanides2017aixijs} which has previously been used to compare asymptotically optimal agents \cite{aslanides2017universal}. We tested in $10\ \times\ 10$ grid-worlds, and  $20\ \times\ 20$ grid-worlds, both with a single dispenser with probability of dispensing reward $0.75$; that is, if the agent enters that cell, the probability of a reward of 1 is 0.75. Following the conventions of \cite{aslanides2017universal} we averaged over 50 simulations, used discount factor $\gamma = 0.99$, 600 MCTS samples, and planning horizon of 6. The planning horizon restricts $m$, and the number of MCTS samples is an input to $\rho$UCT \cite{silver2010monte}, which we use instead of expectimax. The algorithm for the approximate version of Inq is in Appendix \ref{app:inqapprox}.
The code used for this experiment is available online at \url{https://github.com/ejcatt/aixijs}, and this version of Inq can be run in the browser at \url{https://ejcatt.github.io/aixijs/demo.html#inq}.
We found that using small values for $\eta$, specifically $\eta \le 1$ worked well. For our experiments we chose $\eta=1$.



In the $10\times 10$ grid-worlds Inq performed comparably to both BayesExp and Thompson sampling. However in the $20\times 20$ grid-worlds Inq performed comparably to BayesExp, and outperformed Thompson sampling. This is likely because when the Thomspon Sampling Agent samples an environment with a reward dispenser that is inaccessible within its planning horizon, the agent acts randomly rather than seeking new cells. This is contrast to Inq and BayesExp which always have an incentive to explore the frontier of cells that have not been visited. This is especially relevant in the larger grid where the Thomspon sampling agent is more likely to act as if the dispenser is deep in uncharted territory, rather than nearby. In a grid-world, good exploration is just about visiting new states, which both Inq and BayesExp successfully seek.

\begin{figure}[H]
\includegraphics[width=\linewidth]{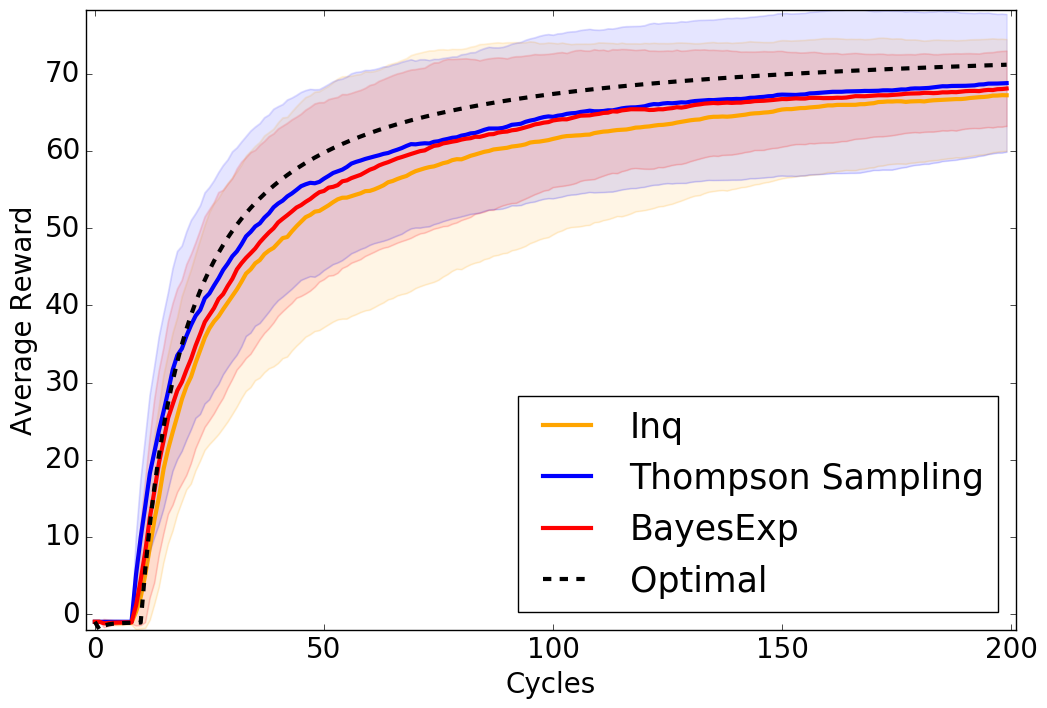}
\caption{$10\times 10$ Grid-worlds}
\end{figure}

\begin{figure}[H]
\includegraphics[width=\linewidth]{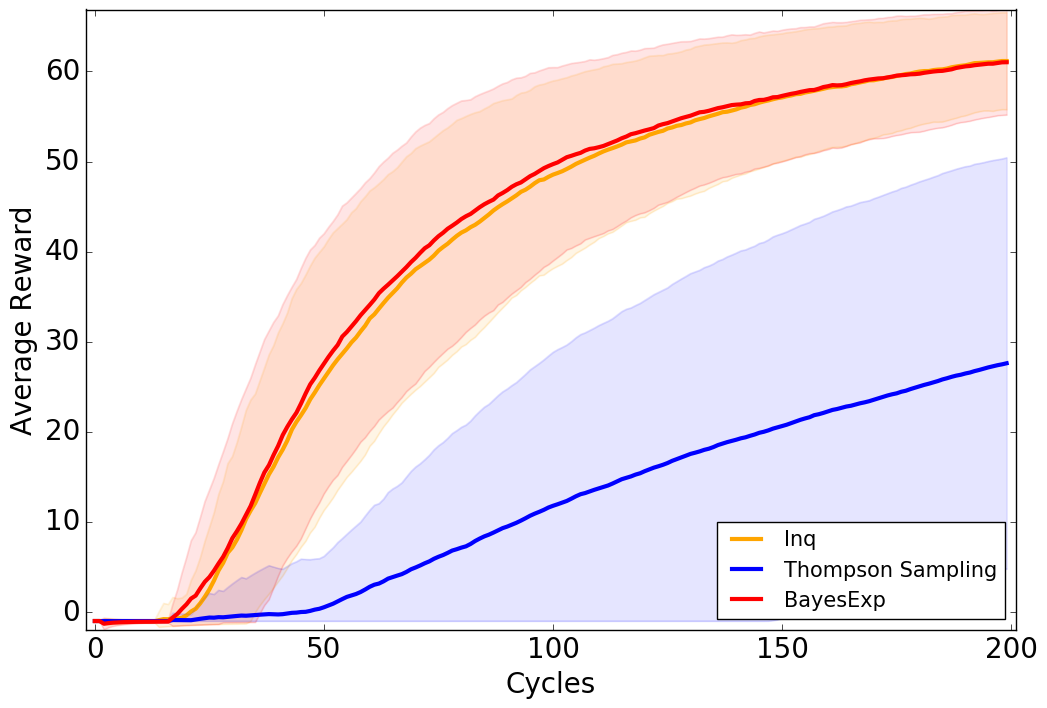}
\caption{$20\times 20$ Grid-worlds}
\end{figure}

\section{Conclusion}

We have shown that it is possible for an agent with a bounded horizon to be strongly asymptotically optimal in the class of all computable environments. No existing RL agent has as strong an optimality guarantee as Inq. The nature of the exploration regime that accomplishes this is perhaps of wider interest. We formalize an agent that gets distracted from reward maximization by its inquisitiveness: the more it expects to learn from an expedition, the more inclined it is to take it.

We have confirmed experimentally that inquisitiveness is a practical and effective exploration strategy for Bayesian agents with manageable model classes.

There are two main avenues for future work we would like to see. The first regards possible extensions of inquisitiveness: we have defined inquisitiveness for Bayesian agents with countable model-classes, but inquisitiveness could also be defined for a Bayesian agent with a continuous model class, such as a Q-learner using a Bayesian Neural Network. The second avenue regards the theory of strong asymptotic optimality itself: is Inq strongly asymptotically optimal for more farsighted discounters? If not, can it be modified to accomplish that? Or is it indeed impossible for an agent with an unbounded horizon to be strongly asymptotically optimal in the class of computable environments? Answers to these questions, besides being interesting in their own right, will likely inform the design of tractable exploration strategies, in the same way that this work has done.

\section*{Acknowledgements}
This work was supported by the Open Philanthropy Project AI Scholarship and the Australian Research Council Discovery Projects DP150104590.

\vfill
\clearpage
\bibliography{cohen}

\onecolumn
\appendix
\appendixpage
\section{Definitions and Notation -- Quick Reference} \label{app:notation}

\begin{equation*}
\begin{split}
\mathcal{H} &:= \mathcal{A} \times \mathcal{O} \times \mathcal{R}
\\
&\hspace*{-20pt} \left.\begin{aligned}
h_{<t} &\in \mathcal{H}^{t-1}
\\
h_{t:k} &\in \mathcal{H}^{k-t+1}
\\
\mu, \nu &\in \mathcal{M}
\\
\mu, \nu &: \mathcal{H}^* \times \mathcal{A} \rightsquigarrow \mathcal{O} \times \mathcal{R}
\\
\pi &: \mathcal{H}^* \rightsquigarrow \mathcal{A}
\end{aligned}\right\rbrace \textrm{typical meaning of certain notation}
\\
\p^\pi_\nu(\epsilon) &:= 1; \ \ \ \p^\pi_\nu(haor) := \p^\pi_\nu(h) \pi(a|h) \nu(or|ha)
\\
w(\nu) &:= 2^{-K(\nu)}
\\
\p^\pi_\xi(\cdot) &:= \sum_{\nu \in \mathcal{M}} w(\nu) \p^\pi_\nu(\cdot)
\\
w(\nu|h) &:= w(\nu) \frac{\p^\pi_\nu(h)}{\p^\pi_\xi(h)}
\\
w_{\inf}(\mu|h_{1:\infty}) &:= \inf_{k \in \mathbb{N}} w(\mu|h_{<k})
\\
\xi(or|ha) &:= \sum_{\nu \in \mathcal{M}} w(\nu|h) \nu(or|ha)
\\[-3pt]
\IG(h_{t:t+k-1}|h_{<t}) &:= \sum_{\nu \in \mathcal{M}} w(\nu|h_{<t+k-1})\log \frac{w(\nu|h_{<t+k-1})}{w(\nu|h_{<t})} \ \ \ \textrm{for} \ h_{<t} \in \mathcal{H}^{t-1}, h' \in \mathcal{H}^k
\\[-3pt]
V^{\IG}(\alpha^m, h_{<t}) &:= \sum_{h_{t:t+m-1}\in \mathcal{H}^m} \p^{\alpha^m}_\xi(h_{<t+m}|h_{<t}) \IG(h_{t:t+m-1}|h_{<t})
\\[-3pt]
\alpha^{\IG}_{m, k}(h_{<t}) &:= \argmax_{\alpha^m \ : \ \bigcup_{i=0}^{m-1} \mathcal{H}^i \to \mathcal{A}} V^{\IG}(\alpha^m, h_{<t-k})
\\
a^{\IG}_{m, k}(h_{<t}) &:= \alpha^{\IG}_{m, k}(h_{<t})(h_{t-k:t-1})
\\[-3pt]
V^\pi_\nu(h_{<t}) &:= \frac{1}{\Gamma_t}\mathbb{E}^\pi_\nu[\sum_{k=t}^\infty \gamma_k r_k|h_{<t}]
\\
a^*(h_{<t}) &:= \pi^*_{\xi}(h_{<t})
\\[-3pt]
V^*_\nu(h_{<t}) &:= \sup_{\pi \in \Pi} V^\pi_\nu(h_{<t}) = V^{\pi^*_{\nu}}_\nu(h_{<t})
\\[-2pt]
\rho(h_{<t},m,k) &:= \max\{\frac{1}{m^2(m+1)}, \eta V^{\IG}(\alpha^{\IG}_{m, k}(h_{<t}), h_{<t-k})\}
\\
\pi^\dagger(a|h_{<t}) &:= \sum_{m \in \mathbb{N}} \sum_{k < m, t} \rho(h_{<t},m,k)[[a = a^{\IG}_{m, k}(h_{<t})]] + (1 - \beta(h_{<t},m,k))[[a = a^*(h_{<t})]]
\\
\KL_{h_{<t}, n}(\p^\pi_{\nu_1} || \p^\pi_{\nu_2}) &:= \sum_{h' \in \mathcal{H}^n} \p^\pi_{\nu_1}(h'|h_{<t}) \log \frac{\p^\pi_{\nu_1}(h'|h_{<t})}{\p^\pi_{\nu_2}(h'|h_{<t})}
\end{split}
\end{equation*}

\newpage

\section{Proofs of Lemmas} \label{app:proofs}

\vspace{12pt}

We begin with a lemma that equates the information gain value of an expedition with the expected prediction error. The KL-divergence on the right hand side represents how different $\nu$ and $\xi$ appear when following the expedition in question.

\lemorseau*

\begin{proof}
This result is shown in \cite[Equation 4]{Hutter:13ksaprob}.
\end{proof}

Recall that $w(\nu|h_{<t})$ is the posterior weight that Inq assigns to the environment $\nu$ after observing $h_{<t}$. We show that the infimum of this value is strictly positive with probability 1.

\lemwmuinf*

\begin{proof}
Suppose $\inf_t w(\mu|h_{<t}) = 0$. $w(\mu|h_{<t}) > 0$ for all histories generated by $\p^\pi_\mu$. Therefore, $\inf_t w(\mu|h_{<t}) = 0 \implies \liminf_{t \to \infty} w(\mu|h_{<t}) = 0$, and $\limsup_{t \to \infty} w(\mu|h_{<t})^{-1} = \infty$. We show that this has probability $0$.

Let 
\begin{equation}
    z_t := w(\mu|h_{\leq t})^{-1} = \frac{\p^\pi_\xi(h_{\leq t})}{\p^\pi_\mu(h_{\leq t})} w(\mu)^{-1}
\end{equation}
I first show that $z_t$ is a $\mu$-martingale.

\begin{align*}
\mathbb{E}^\pi_\mu[z_t | h_{<t}] &= w(\mu)^{-1} \sum_{h_t \in \mathcal{H}} \p^\pi_\mu(h_{t} | h_{<t}) \frac{\p^\pi_\xi(h_{\leq t})}{\p^\pi_\mu(h_{\leq t})}
\\
&= w(\mu)^{-1} \frac{\sum_{h_t \in \mathcal{H}} \p^\pi_\xi(h_{\leq t})}{\p^\pi_\mu(h_{<t})}
\\
&= w(\mu)^{-1} \frac{\p^\pi_\xi(h_{<t})}{\p^\pi_\mu(h_{<t})}
\\
&= z_{t-1}
\tagaligneq
\end{align*}

By the martingale convergence theorem $z_t \to f(\omega) < \infty \ \ \mathrm{w.p. 1}$, for $\omega \in \Omega$, the sample space, and some $f: \Omega \to \mathbb{R}$. Therefore, $\inf_t w(\mu|h_{<t}) > 0 \ \ \mathrm{w.p. 1}$.
\end{proof}

Next we show that every exploration probability $\rho(h_{<t}, m, k)$ goes to $0$.
From here, all ``w.p.1'' statements mean with $\p^{\pi^\dagger}_\mu$-probability 1, if not otherwise specified.

\lemexppart*

\begin{proof}
$\rho(h_{<t}, m, k) = \rho(h_{<t-k}, m, 0)$, so we need only show that $\rho(h_{<t}, m, 0) \to 0$ w.p.1. We do this by showing that the expectation of $\rho(h_{<t}, m, 0)^{m+1}$ is summable. (This is a stronger result, since it implies that it is summable with probability 1, so the probability that it is greater than $\varepsilon$ infinitely often is 0.) A bit of notational background: $0 \in \mathbb{N}$, and $m \mathbb{N} + i = \{i, i+m, i+2m, ...\}$. Each equation and inequality is explained below.

\begin{align*} \label{eqn:limexp}
    &\ \ \ \ \ w(\mu) \E^{\pi^\dagger}_\mu \sum_{t \in m\mathbb{N} + i} \rho(h_{<t}, m, 0)^{m+1}
    \\
    &\lequal^{(a)} \sum_{\nu \in \mathcal{M}} w(\nu) \E^{\pi^\dagger}_\nu \sum_{t \in m\mathbb{N} + i} \rho(h_{<t}, m, 0)^{m+1}
    \\
    &\equal^{(b)} \E^{\pi^\dagger}_\xi \sum_{t \in m\mathbb{N} + i} \rho(h_{<t}, m, 0)^{m+1}
    \\
    &\lequal^{(c)} \E^{\pi^\dagger}_\xi \sum_{t \in m\mathbb{N} + i} \rho(h_{<t}, m, 0)^m \eta V^{\IG}(\alpha^{\IG}_{m, 0}, h_{<t})
    \\
    &\equal^{(d)} \eta \sum_{t \in m\mathbb{N} + i} \E_{h_{<t} \sim \p^{\pi^\dagger}_\xi} \left[\rho(h_{<t}, m, 0)^m \E_{h_{t:t+m-1} \sim \p^{\alpha^{\IG}_{m, 0}}_\xi} \left[ \IG(h_{t:t+m-1} | h_{<t}) \right]\right]
    \\
    &\equal^{(e)} \eta \sum_{t \in m\mathbb{N} + i} \E_{h_{<t} \sim \p^{\pi^\dagger}_\xi} \left[\sum_{h_{t:t+m-1} \in \mathcal{H}^m} \rho(h_{<t}, m, 0)^m \p^{\alpha^{\IG}_{m, 0}}_\xi(h_{t:t+m-1}) \left[ \IG(h_{t:t+m-1} | h_{<t}) \right]\right]
    \\
    &\lequal^{(f)} \eta \sum_{t \in m\mathbb{N} + i} \E_{h_{<t} \sim \p^{\pi^\dagger}_\xi} \left[\sum_{h_{t:t+m-1} \in \mathcal{H}^m} \p^{\pi^\dagger}_\xi(h_{t:t+m-1}) \left[ \IG(h_{t:t+m-1} | h_{<t}) \right]\right]
    \\
    &\equal^{(g)} \eta \sum_{t \in m\mathbb{N} + i} \E^{\pi^\dagger}_\xi \IG(h_{t:t+m-1} | h_{<t})
    \\
    &\equal^{(h)} \eta \E^{\pi^\dagger}_\xi \sum_{t \in m\mathbb{N} + i} \sum_{\nu \in \mathcal{M}} w(\nu | h_{<t+m}) \log \frac{w(\nu | h_{<t+m})}{w(\nu | h_{<t})}
    \\
    &\equal^{(i)} \eta \sum_{t \in m\mathbb{N} + i} \sum_{\nu \in \mathcal{M}} \E^{\pi^\dagger}_\xi \frac{w(\nu) \nu(o\!r_{<t+m}|a_{<t+m})} {\xi(o\!r_{<t+m}|a_{<t+m})} \log \frac{w(\nu | h_{<t+m})}{w(\nu | h_{<t})}
    \\
    &\equal^{(j)} \eta \sum_{t \in m\mathbb{N} + i} \sum_{\nu \in \mathcal{M}} \E^{\pi^\dagger}_\nu w(\nu) \log \frac{w(\nu | h_{<t+m})}{w(\nu | h_{<t})}
    \\
    &\equal^{(k)} \lim_{N \to \infty} \eta \sum_{k = 0}^{N-1} \sum_{\nu \in \mathcal{M}} \E^{\pi^\dagger}_\nu w(\nu) \log \frac{w(\nu | h_{<mk+i+m})}{w(\nu | h_{<mk+i})}
    \\
    &\equal^{(l)} \lim_{N \to \infty} \eta \sum_{\nu \in \mathcal{M}} \E^{\pi^\dagger}_\nu w(\nu) \log \prod_{k = 0}^{N-1} \frac{w(\nu | h_{<m(k+1)+i})}{w(\nu | h_{<mk+i})}
    \\
    &\equal^{(m)} \lim_{N \to \infty} \eta \sum_{\nu \in \mathcal{M}} \E^{\pi^\dagger}_\nu w(\nu) \log \frac{w(\nu | h_{<mN+i})}{w(\nu | h_{<i})}
    \\
    &\lequal^{(n)} \eta \sum_{\nu \in \mathcal{M}} \E^{\pi^\dagger}_\nu w(\nu) \log \frac{1}{w(\nu | h_{<i})}
    \\
    &\equal^{(o)} \eta \sum_{\nu \in \mathcal{M}} \E^{\pi^\dagger}_\nu w(\nu) \log \frac{1}{w(\nu)} \frac{w(\nu)}{w(\nu | h_{<i})}
    \\
    &\equal^{(p)} \eta \sum_{\nu \in \mathcal{M}} w(\nu) \log \frac{1}{w(\nu)} + \eta \sum_{\nu \in \mathcal{M}} \E^{\pi^\dagger}_\nu w(\nu) \log \frac{w(\nu)}{w(\nu | h_{<i})}
    \\
    &\equal^{(q)} \eta \mathrm{Ent}(w) + \eta \sum_{h_{<i} \in \mathcal{H}^i} \sum_{\nu \in \mathcal{M}} w(\nu) \p^{\pi^\dagger}_\nu(h_{<i}) \log \frac{w(\nu)}{w(\nu | h_{<i})}
    \\
    &\equal^{(r)} \eta \mathrm{Ent}(w) + \eta \sum_{h_{<i} \in \mathcal{H}^i} \sum_{\nu \in \mathcal{M}} w(\nu | h_{<i}) \p^{\pi^\dagger}_\xi(h_{<i}) \log \frac{w(\nu)}{w(\nu | h_{<i})}
    \\
    &\equal^{(s)} \eta \mathrm{Ent}(w) - \eta \E^{\pi^\dagger}_\xi [\IG(h_{<i}|\epsilon)] \lequal^{(t)} \eta \mathrm{Ent}(w) \lthan^{(u)} \infty
    \tagaligneq
\end{align*}
For multiple steps in this derivation, note that the information gain is non-negative; this is a property of the KL-divergence. (a) follows from the l.h.s. being one of the non-negative summands of the r.h.s. (b) follows from the definition of $\xi$. (c) follows from the definition of $\rho$. (d) substitutes $V^{\IG}$ for its definition. (e) expands the definition of the expectation. (f) follows because $\pi^\dagger$ mimics $\alpha^{\IG}_{m, 0}$ for $m$ consecutive timesteps with probability $\prod_{i = 0}^{m-1} \rho(h_{<t+i}, m, i) = \rho(h_{<t}, m, 0)^m$, so the probability of any history under $\p^{\pi^\dagger}_\xi$ is at least the probability of that history under $\p^{\alpha^{\IG}_{m, 0}}_\xi$ times $\rho(h_{<t}, m, 0)^m$. (g) combines the two expectations, which are now with respect to the same probability measure. (h) expands the definition of the information gain. (i) rearranges the expectations and the sums, and expands $w(\nu | h_{<t+m})$ according to Bayes' rule. (j) converts the expectation to a expectation with respect to a different probability measure through simple cancellation. (k) implements a change of variable from $t$ to $mk + i$. (l) moves a sum inside the logarithm. (m) cancels out all terms expect the numerator of the last term and the denominator of the first. (n) follows from all posterior weights being $\leq 1$. (o) and (p) are obvious. (q) applies the definition of the entropy of a distribution $\mathrm{Ent}(\cdot)$, and expands the expectation. (r) changes the variable in the expectation; this is the reverse of (i) and (j). (s) applies the definition of the information gain (after inverting the fraction in the logarithm). (t) follows from the non-negativity of the information gain. And (u) is shown in \cite[Proposition 13]{Hutter:13ksaprob}.

Finally, 
\begin{align*} \label{eqn:removemod}
    \E^{\pi^\dagger}_\mu \sum_{t = 0}^\infty \rho(h_{<t}, m, 0)^{m+1} &= \sum_{i = 0}^{m-1} \E^{\pi^\dagger}_\mu \sum_{t \in m\mathbb{N} + i} \rho(h_{<t}, m, 0)^{m+1}
    \\
    &\lequal^{(\ref{eqn:limexp})} \sum_{i = 0}^{m-1} \frac{\eta \mathrm{Ent}(w)}{w(\mu)} = \frac{m \eta \mathrm{Ent}(w)}{w(\mu)} < \infty
    \tagaligneq
\end{align*}
\end{proof}

Now, we show that the total exploration probability goes to 0:
\lemexp*

\begin{proof}
\begin{align*}
    \beta(h_{<t}) &= \sum_{m \in \mathbb{N}} \sum_{k = 0}^{\min\{m-1, t\}} \rho(h_{<t}, m, k)
    \\
    &= \sum_{m \in \mathbb{N}} \sum_{k = 0}^{\min\{m-1, t\}} \min \left\{\frac{1}{m^2(m+1)},V^{\IG}_{m, k}(h_{<t})\right\}
    \tagaligneq
\end{align*}

Each of the terms in the sum approaches $0$ with probability 1 by Lemma \ref{lemexppart0}, and because $\rho(h_{<t}, m, k) = \rho(h_{<t - k}, m, 0)$. Suppose by contradiction $\beta(h_{<t}) > \varepsilon > 0$ infinitely often. There exists an $M$ such that
\begin{equation}
    \sum_{m = M}^\infty \sum_{k = 0}^{\min\{m-1, t\}} \rho(h_{<t}, m, k) < \sum_{m = M}^\infty \sum_{k = 0}^{m-1} \frac{1}{m^2(m+1)} < \varepsilon/2
\end{equation}
for all $t$. With that $M$, then if $\beta(h_{<t}) > \varepsilon$ infinitely often, it must be the case that $\sum_{m = 0}^{M-1} \sum_{k = 0}^{m-1} \rho(h_{<t}, m, k) > \varepsilon/2$ infinitely often, but this is a finite sum of terms that all approach $0$, a contradiction.
\end{proof}

Lemma \ref{lempredconv} shows that the probabilities assigned by $\xi$ converge to those of $\mu$.

\lempredconv*

\begin{proof}
Suppose that $0 < \varepsilon \leq (\p^{\alpha^m}_\mu(h_{t:t+m-1}|h_{<t}) - \p^{\alpha^m}_\xi(h_{t:t+m-1}|h_{<t}))^2$ for some $h_{t:t+m-1}$.
\begin{align*} \label{eqn:useentropyineq}
    \varepsilon &\leq (\p^{\alpha^m}_\mu(h_{t:t+m-1}|h_{<t}) - \p^{\alpha^m}_\xi(h_{t:t+m-1}|h_{<t}))^2
    \\
    &\lequal^{(a)} \KL_{h_{<t}, m}\left(\p^{\alpha^m}_\mu \vb \vb \p^{\alpha^m}_\xi \right)
    \\
    &\lequal^{(b)} \sum_{\nu \in \mathcal{M}} \frac{w(\nu | h_{<t})}{w(\mu | h_{<t})} \KL_{h_{<t}, m}\left(\p^{\alpha^m}_\nu \vb \vb \p^{\alpha^m}_\xi \right)
    \\
    &\equal^{(c)} \frac{1}{w(\mu | h_{<t})} V^{\IG}(\alpha^m, h_{<t})
    \\
    &\lequal^{(d)} \frac{1}{\inf_k w(\mu | h_{<k})} V^{\IG}(\alpha^m, h_{<t})
    \\
    &\lequal^{(e)} \frac{1}{\inf_k w(\mu | h_{<k})} V^{\IG}(\alpha^{\IG}_{m, 0}(h_{<t}), h_{<t})
    \tagaligneq
\end{align*}
(a) is a result from information theory known as the entropy inequality. (b) follows from the non-negativity of the KL-divergence, and the l.h.s. being one of the summands of the r.h.s. (c) follows from Lemma~\ref{lemorseau4}. (d) follows from the definition of the infimum. And (e) follows from the fact that $\alpha^{\IG}_{m, 0}(h_{<t})$ maximizes $V^{\IG}(\cdot, h_{<t})$, by definition.

Therefore,
\begin{align*}
    &(\p^{\alpha^m}_\mu(h_{t:t+m-1}|h_{<t}) - \p^{\alpha^m}_\xi(h_{t:t+m-1}|h_{<t}))^2 \geq \varepsilon \ \ \mathrm{i.o.}
    \\
    \textrm{implies} \ \ &V^{\IG}(\alpha^{\IG}_{m, 0}(h_{<t}), h_{<t}) \geq \varepsilon \inf_k w(\mu | h_{<k}) \ \ \mathrm{i.o.}
    \\
    \textrm{which implies} \ \ &\rho(h_{<t}, m, 0) \geq \min\{\frac{1}{m^2(m+1)}, \varepsilon \inf_k w(\mu | h_{<k})\} \ \ \mathrm{i.o.}
    \\
    \textrm{which implies} \ \ &\sum_{t = 0}^\infty \rho(h_{<t}, m, 0)^{m+1} = \infty \ \ \textrm{or} \ \ \inf_k w(\mu | h_{<k}) = 0\ 
\end{align*}

This has probability 0 by Lemmas \ref{lemexppart0} and \ref{lemwmuinf}. Thus, with probability 1, $\p^{\alpha^m}_\mu(h_{t:t+m-1}|h_{<t}) - \p^{\alpha^m}_\xi(h_{t:t+m-1}|h_{<t}) \to 0$.

\end{proof}

\section{Approximation of Inq} \label{app:inqapprox}
Following Aslanides \shortcite{aslanides2017aixijs}, our approximation of Inq calls $\rho$UCT \cite{silver2010monte} as a subroutine in place of expectimax.

\begin{algorithm}
\caption{Approximation of Inquisitive Reinforcement Learner's Policy}
\begin{algorithmic}[1]
\Require MCTS Samples, horizon, $\gamma$
\renewcommand{\algorithmicrequire}{\textbf{Initialize:}}
\Require uniform prior over model class
\While{True}

\ForAll {$m \leq$ horizon and $k < \min\{m, t\}$}
\State using information gain as reward, $a^{\IG}_{m, k} \sim \rho\textrm{UCT}(h_{<t-k}, \textrm{MCTS samples}, m, \gamma)$
\State {$\rho (m,k)$ = $\min\{$information-gain-value of $a^{\IG}_{m, k}$, $1/(m^2(m+1))\}$}
\EndFor
\State using the actual reward, $a^* \sim \rho\textrm{UCT}(h_{<t}, \textrm{MCTS samples}, \textrm{horizon}, \gamma)$
\State take action $a^{\IG}_{m, k}$ with probability $\rho (m,k)$ for all $m \leq$ horizon and $k < \min\{m, t\}$ else take action $a^*$
\State update posterior from observation and reward

\EndWhile
\end{algorithmic}
\end{algorithm}

\end{document}